\documentclass{article}

\usepackage[preprint]{neurips_2021}

\usepackage[utf8]{inputenc} 
\usepackage[T1]{fontenc}    
\usepackage{hyperref}       
\usepackage{url}            
\usepackage{booktabs}       
\usepackage{amsfonts}       
\usepackage{nicefrac}       
\usepackage{microtype}      
\usepackage{xcolor}         

\usepackage{amssymb}
\usepackage{amsmath}
\usepackage{mathtools}
\usepackage{empheq}
\usepackage{dsfont}
\usepackage{amsthm}
\usepackage{bbm}
\usepackage{enumerate}
\usepackage{bm}

\usepackage{algorithm}
\usepackage{algorithmic}

\newcommand{\N}{\mathcal{N}}
\newcommand{\GP}{\mathcal{GP}}
\newcommand{\bigO}{\mathcal{O}}

\newtheorem{theorem}{Theorem}[section]

\newtheorem{lemma}[theorem]{Lemma}

\newcommand{\john}[1]{{{\textcolor{black}{#1}}}}
\title{The Fast Kernel Transform}

\author{%
  John Paul Ryan\\
  Department of Computer Science\\
  Cornell University\\
  Ithaca, NY 14853\\
  \texttt{johnryan@cs.cornell.edu} \\
   \And
   Sebastian Ament \\
   Department of Computer Science \\
   Cornell University \\
     Ithaca, NY 14853\\
   \texttt{sea79@cornell.edu} \\
   \AND
   Carla P. Gomes \\
   Department of Computer Science \\
   Cornell University \\
     Ithaca, NY 14853\\
   \texttt{gomes@cs.cornell.edu} \\
   \And
   Anil Damle \\
   Department of Computer Science \\
   Cornell University \\
    Ithaca, NY 14853\\
   \texttt{damle@cornell.edu} 
}

\begin{document}

\maketitle

\begin{abstract}
Kernel methods are a highly effective and widely used collection of modern machine learning algorithms.
A fundamental limitation of virtually all such methods are computations involving the kernel matrix that na\"ively scale quadratically (e.g., constructing the kernel matrix and matrix-vector multiplication) or cubically (solving linear systems) with the size of the data set $N.$ 
We propose the Fast Kernel Transform (FKT), a general algorithm to compute matrix-vector multiplications (MVMs) for datasets in moderate dimensions with quasilinear complexity.
Typically, analytically grounded fast multiplication methods require specialized development for specific kernels. In contrast, our scheme is based on auto-differentiation and automated symbolic computations that leverage the analytical structure of the underlying kernel. 
This allows the FKT to be easily applied to a broad class of kernels, including Gaussian, Mat\'ern, and Rational Quadratic covariance functions and 
physically motivated Green's functions, including those of the Laplace and Helmholtz equations. Furthermore, the FKT maintains a high, quantifiable, and controllable level of accuracy---properties that many acceleration methods lack.
We illustrate the efficacy and versatility of the FKT by providing timing and accuracy benchmarks and by applying it to scale the stochastic neighborhood embedding (t-SNE) and Gaussian processes to large real-world data sets.
\end{abstract}

\section{Introduction}

Kernel methods are fundamental to machine learning and many of its applications. Examples include kernel density estimation,
kernel regression, Gaussian processes, support vector machines,
kernel clustering, and kernel PCA~\citep{shawe2004kernel, scholkopf2018learning}.
While these methods are highly expressive by computing with an infinite-dimensional feature space using the ``kernel trick,''
most methods require solving linear systems with the kernel matrix---an operation that scales cubically with the number of data points.
This is prohibitively expensive for increasingly large modern data sets and fundamentally limits the applicability of kernel methods.

To remedy this, a large number of methods have been developed that accelerate operations involving kernel matrices. Typically, these methods provide faster matrix vector products and may be paired with classical iterative methods to solve the necessary linear systems. For example, in the machine learning community, a popular approach is the Nystr\"om method, which constructs a low-rank approximation based on a random sample of a kernel matrix's columns \citep{williams2001using, drineas2005nystrom, kumar2009ensemble, kumar2012sampling}.
In the context of Gaussian Process (GP) regression, \citet{snelson2005sparse} introduced inducing inputs, leading to $\mathcal{O}(Nm^2)$ runtime for $N$ data points and $m$ inducing inputs.
In Section~\ref{sec:methods} we develop a new scheme for this problem based on analytical expansions which can be readily applied to a broad range of kernel functions that arise in a diverse set of applications---a feature we highlight in Section~\ref{sec:experiments}.

In scientific computing and applied mathematics, a large body of work concerns the acceleration of physical simulations in which the force two particles exert on each other is modeled by a kernel function,
like the inverse-square law $\sim 1/\|\mathbf{x} - \mathbf{ y}\|^2$ for gravitational and electromagnetic forces.
Famously,
\citet{greengard1987fast} introduced the Fast Multipole Method (FMM), which provides  linear-time computation of approximate matrix-vector multiplications with certain Green's function kernel matrices based on analytical expansions. The Fast Gauss Transform (FGT) \citep{greengard1991fast} 
applied similar analysis to the Gaussian kernel, and was subsequently improved to enable efficient computations in higher dimensions~\citep{yang2003improved}
and applied to kernel-based machine learning methods
\citep{yang2004efficient}.
Importantly, in these cases it is possible to derive concrete error bounds based on the analytical expansions.
However, extending these methods relies on extensive work per kernel and is dependent on finding/developing appropriate analytical expansions. In contrast, our method leverages a new general analytical expansion to allow for immediate application to a variety of kernels. Even with this generality, we are still able to provide bounds and computational complexity analysis in Section~\ref{sec:analysis} that is experimentally demonstrated in Section~\ref{sec:experiments}. 

\paragraph{Contribution}
In this work, we propose the Fast Kernel Transform (FKT),
an algorithm that allows for matrix-vector multiplication with kernel matrices in $\mathcal{O}(N \log N)$ operations and is applicable to any isotropic kernel which is analytic away from the origin and any dataset in moderate dimensions.
The FKT achieves this combination of computational efficiency and broad applicability by leveraging a new general analytical expansion introduced herein, which is implemented in Julia using modern computer algebra and auto-differentiation technologies and is provided open-source. 
We demonstrate the FKT's scaling on synthetic data and apply it to stochastic neighborhood embedding (t-SNE) and Gaussian process regression using real-world oceanographic data to highlight the method's versatility.

\section{Prior Work}
\label{sec:prior}
Algorithms that compute (approximate) matrix vector products with kernel matrices have a long history and include algorithms of various flavors. Simplistically, these methods either leverage a regular grid in the underlying domain or adaptive decompositions, and either use analytical expansions for kernel functions or purely computational schemes for compression. Concretely, our FKT leverages adaptive decompositions and a semi-analytic scheme for compressing long-range interactions.

\paragraph{Adaptive Methods}
The need for fast summation methods in N-body problems for unstructured data (i.e., matrix vector products with specific kernels) drove the development of methods that take advantage of two key features: (1) adaptive decompositions of the underlying spatial domain and (2) the ability to compress interactions between points that are well separated. This led to the development of the Barnes-Hut algorithm~\citep{barnes1986hierarchical} and the FMM~\citep{greengard1987fast} for computing matrix vector products. While the FMM attains $\bigO(N)$ scaling (with a constant that depends mildly on the desired accuracy), it explicitly leverages an analytical expansion for the underlying kernel and associated translation operators. Therefore, extending the algorithm to additional kernels requires extensive work. This has been done for e.g. the Helmholtz kernel via the use of Bessel and Hankel functions~\citep{fmmhelm}.

To expand the applicability of these adaptive methods to more general kernels, numerical schemes were developed to compress long-range kernel interactions. These schemes led to algorithms such as the kernel independent FMM~\citep{ying2004kernel,ying2006kernel} and, more generally, so-called rank-structured factorizations and fast direct methods for matrices (see, e.g.,~\citep{martinsson2019fast} for an overview of these methods in the context of integral equations). Moreover, these methods have been successfully applied to Gaussian Process regression~\citep{borm2007approximating,ambikasaran2015fast,minden2017fast}. While broadly applicable, these methods can be sub-optimal if analytical expansions for kernel functions are available, as they rely on algebraic factorizations such as the interpolative decomposition~\citep{idfact}.

\paragraph{Grid-Based Methods}
For certain data distributions it can be advantageous to leverage regular grids on the computational domain to accelerate matrix vector products (and/or build effective pre-conditioners). Notably, if the observation points lie on a regular grid and the kernel function has certain structural properties it is possible to leverage the Fast Fourier Transform (FFT) to compute matrix vector products in $\bigO(N\log N)$ time. However, observation points typically do not lie precisely on a regular grid. The so-called pre-corrected FFT~\citep{phillips1994precorrected,white1994comparing} solves this problem by incorporating aggregation and interpolation operators to allow for computations using a regular grid that are then accelerated by the FFT. An analogous method called structured kernel interpolation (SKI) is popular within the Gaussian Process community~\citep{wilson2015kernel} as an acceleration of the so-called inducing point method~\citep{snelson2005sparse}.



\section{The Fast Kernel Transform}
\label{sec:methods}
We are interested in computing the matrix-vector product
\begin{equation}\label{eq:nbody}
z_i = \sum_{j=0}^N K(|\mathbf{r}_i-\mathbf{r}_j|)y_j.
\end{equation} 
where $y$ is a given vector of real or complex numbers, $\mathbf{r}_i\in \mathbb{R}^d$ for $i=0,\ldots,N$, and $K$ is an isotropic kernel. Henceforth, we will overload notation to say that $K_{ij}\coloneqq K(|\mathbf{r}_i-\mathbf{r}_j|)$ and~\eqref{eq:nbody} can be written as $z=Ky$. The technique we propose is based on the famous Barnes-Hut~\citep{barnes1986hierarchical} style of tree-code algorithm. A tree decomposition is performed of the space containing the dataset's points, and for each tree node, we compute a set of distant points whose kernel interactions with the node's points can be compressed. In the original Barnes-Hut scheme, this compression is done by summing interactions with the center of mass\textemdash in our scheme we generalize this to a new multipole expansion which can more accurately represent the points inside the node. Compressing these interactions will produce low-rank approximations for large off-diagonal blocks of the kernel matrix, yielding an efficient matrix multiplication algorithm. We review each of these components in the following sections.
\subsection{Tree decomposition}

We use a decomposition inspired by the binary partitioning of the $k$-d tree \citep{kdtree}. This scheme begins with a single hypercube root node containing all points, and iteratively splits nodes in half via axis-aligned separating hyperplanes. At each split the hyperplane is chosen to (a) split the node in half, (b) keep the aspect ratio (the maximum ratio between pairs of node side lengths) below two, and (c) optimally divide the points evenly while satisfying the first two constraints. These qualities are chosen to encourage hyperrectangular nodes with minimal aspect ratio while maintaining the divide-and-conquer nature of binary space partitionings commonly applied in this domain. When a node contains fewer than some prescribed threshold of points, it is not split and becomes a leaf node.
An example of this decomposition is shown in Figure~\ref{fig:domain}.

Once a domain decomposition is computed, our algorithm requires, for every tree node $i$, a set $F_i$ of far points which are far enough from the node to allow accurate compression, and such that $F_i\cap F_j=\emptyset$ if $i$ is a descendent of $j$. Throughout this work we use the following condition for `far enough':
\begin{equation}\label{eq:accparam}
\max_{\mathbf{r}'\in \mathrm{node}}|\mathbf{r}'-\mathbf{r}_c|/|\mathbf{r}-\mathbf{r}_c|<\theta
\end{equation}
where $\mathbf{r}_c$ is the center of the relevant node. If $\mathbf{r}$ satisfies this inequality, it is judged to be far enough away for compression. The distance parameter $\theta$ may then be varied to trade-off accuracy and computation time.

\subsection{Fast Matrix-Vector Multiplication}
Once the sets of far points are generated for all nodes, the FKT proceeds as described in Algorithm~\ref{alg:fact}. For each node $i$, we use a low-rank approximation of the kernel to compute interactions between points in the node and those in the $F_i$. Furthermore, for each leaf $l$ we use exact dense computations for interactions between points in the leaf and its nearby points $N_l$, where $N_l$ is defined to be all points such that $N_l\cap F_j=\emptyset$ for all $j$ in the path from the leaf to the root. In summary the approximation is given by 
\[z = Ky = 
\sum_{l\in\mathrm{leaves}} K_{{N_l},l}*y_{l} +
\sum_{b\in\mathrm{nodes}} K_{{F_b},b}*y_{b}
\approx
\sum_{l\in\mathrm{leaves}} K_{{N_l},l}*y_{l} +
\sum_{b\in\mathrm{nodes}} \overline{K}_{{F_b},b}*y_{b},\]
where $K_{N_l,l}$ is the submatrix of $K$ whose columns correspond to points in the leaf node $l$ and whose rows correspond to points in the near field $N_l$ of the leaf node $l$,  $K_{F_b,b}$ is the analogous submatrix for any node $b$ and its far field $F_b$, $y_l$ and $y_b$ are the subvectors of $y$ corresponding to the points in the leaf $l$ or box $b$ respectively, 
and $\overline{K}_{{F_b}b}$ is a low rank approximation to the typically large $K_{{F_b}b}$. In Algorithm~\ref{alg:fact}, $s2m$ and $m2t$ refer to ``source-to-multipole'' and ``multipole-to-target'' matrices respectively, and collectively form the low-rank approximation $\overline{K}_{{F_b}b}$.
\begin{figure}[t]
\hspace*{\fill}
\begin{minipage}{.49\columnwidth}
\begin{figure}[H]
\vskip 0.2in
\centering
\centerline{
\includegraphics[width=.85\columnwidth, height=.8\columnwidth]
{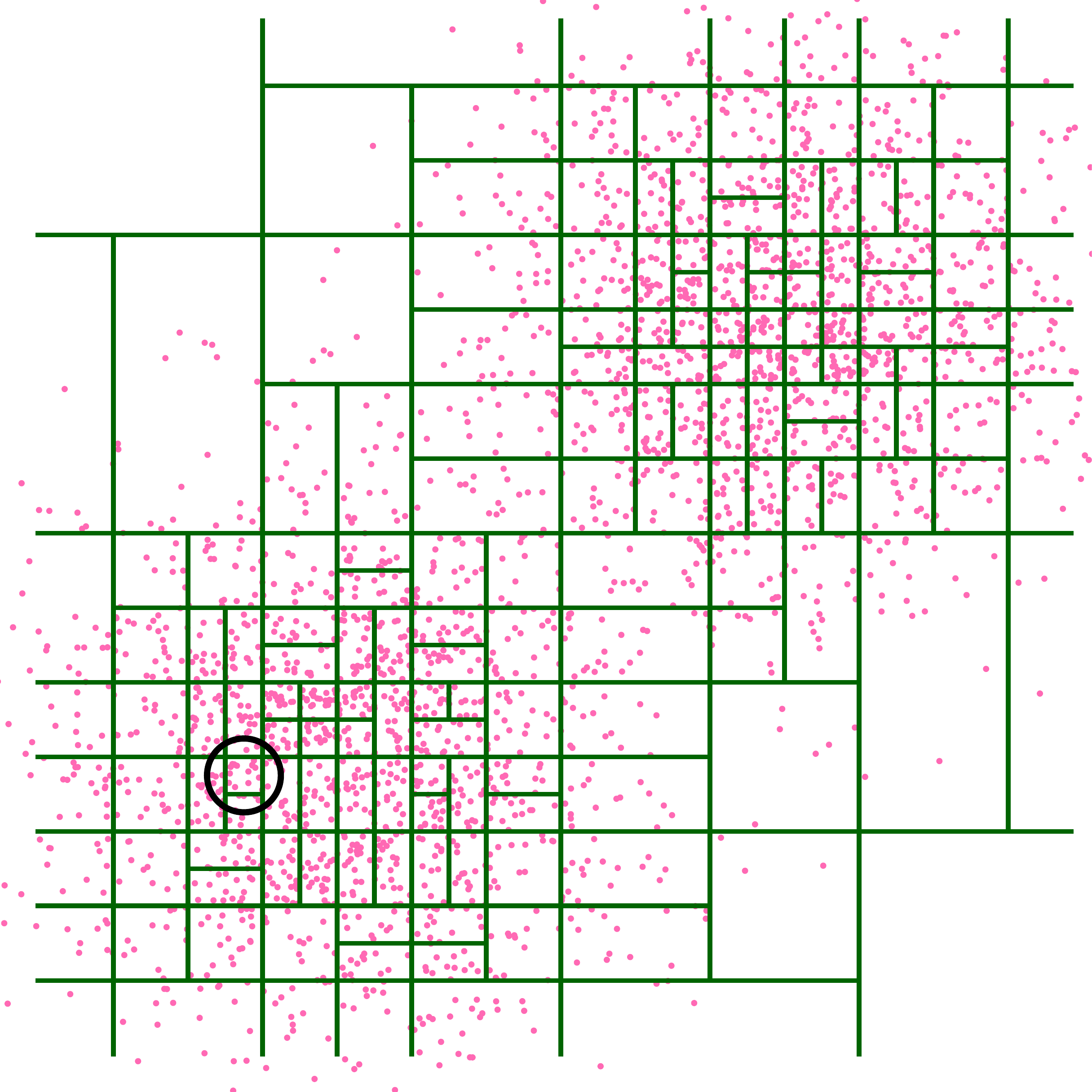}}
\caption{2D domain decomposition on points from a Gaussian mixture. Points outside the circle are considered distant enough for compression with the circled box, for a certain $\theta$ in \eqref{eq:accparam}.}
\label{fig:domain}
\end{figure}
\end{minipage}
\hfill
\begin{minipage}{.48\columnwidth}
\begin{algorithm}[H]
\caption{Barnes-Hut with Multipoles}
\begin{algorithmic}[H]
\STATE{tree $\gets$ BinarySpacePartitioning(points)}
\STATE{$z\gets 0$}
\FOR{$n \in$ tree.nodes}
    \STATE{\COMMENT{Compute compressed far interactions.}}
        \STATE{$s2m$ $\gets$ source2mult($n$)}
        \STATE{$m2t$ $\gets$ mult2target($n$)}
        \STATE{$z$[$n$.far] $\mathrel{{+}{=}} m2t * (s2m * y$[$n$.indices])}
        \IF{isleaf($n$)}
        \STATE{\COMMENT{Compute nearby dense interactions.}}
    \STATE{near\_mat $\gets K(n$.near, $n$.indices)}
    \STATE{$z$[$n$.near] $\mathrel{{+}{=}}$ near\_mat $*y$[$n$.indices]}
    \ENDIF
\ENDFOR
\STATE{\textbf{return} $z$}
\end{algorithmic}\label{alg:fact}
\end{algorithm}
\end{minipage}
\hspace*{\fill}
\end{figure}


\subsection{Low-rank kernel approximations}
Given the preceding approach, the key to a fast algorithm is the availability of a sufficiently accurate low-rank approximation $\overline{K}_{{F_b},b}\approx K_{{F_b},b}$ valid when the sets $F_b$ and $b$ contain well-separated points. Our approach to building these approximations is inspired by multipole methods (specifically the FMM~\citep{greengard1987fast}) for solving the N-body problem~\eqref{eq:nbody} when $K$ is the electrostatic potential.
If $\lvert b \rvert = M$ and $\lvert F_b\rvert=N,$ multiplying by this matrix requires $\mathcal{O}(MN)$ work. However, if we have access to a low rank approximation 
\begin{equation}\label{eq:lowrank}
K(|\mathbf{r}_i-\mathbf{r}_j|) \approx \sum_{k=0}^\mathcal{P} U_k(\mathbf{r}_i)V_k(\mathbf{r}_j)
\end{equation}
valid for $i\in b$ and $j\in F_b$ we can use it to accelerate the computation.
Specifically, using~\eqref{eq:lowrank} we can rewrite \eqref{eq:nbody} as
\[
y_i \approx \sum_{j=0}^N \sum_{k=0}^\mathcal{P} U_k(\mathbf{r}_i)V_k(\mathbf{r}_j)
x_j = 
\sum_{k=0}^\mathcal{P}
U_k(\mathbf{r}_i)
\sum_{j=0}^N  
V_k(\mathbf{r}_j)
x_j 
\]
and the two sums may be computed in $\mathcal{O}(\mathcal{P}(M+N))$ time. In this case, the $V_k$ sum corresponds to the $s2m$ matrix in Algorithm~\ref{alg:fact} and the $U_k$ sum corresponds to the $m2t$ matrix.

For example, let $\mathbf{r'},\mathbf{r}\in\mathbb{R}^3$ with $r' \coloneqq|\mathbf{r'}| < r\coloneqq |\mathbf{r}|$.
A classic example of an expansion of the form in \eqref{eq:lowrank} which is low rank for well-separated points is the multipole expansion of the electrostatic potential
\[
K(|\mathbf{r}'-\mathbf{r}|) = 
\frac{1}{|\mathbf{r}'-\mathbf{r}|}
=\frac{1}{r\sqrt{1+\frac{r'}{r}(\frac{r'}{r}-2\cos{\gamma})}}
\]
where $\gamma$ is the angle between $\mathbf{r'}$ and $\mathbf{r}$. 
Expanding in powers of $\frac{r'}{r}$ yields the expansion in Legendre polynomials
\begin{equation}\label{eq:electro}
K(|\mathbf{r}'-\mathbf{r}|) =\frac{1}{r} \sum_{k=0}^\infty \left(\frac{r'}{r}\right)^kP_k(\cos{\gamma}). 
\end{equation}
This may be put into the form of \eqref{eq:lowrank} by splitting $P_k(\cos{\gamma})$ into functions of $\mathbf{r'}$ and $\mathbf{r}$ using the spherical harmonic addition theorem (see Sec. 12.8 in \citep{spherharm}).

 \begin{equation}\label{eq:spheradditiontheorem}
 \frac{2k+1}{4\pi}P_k(\cos{\gamma})
 = \sum_{h=-k}^k
 Y_k^h(\mathbf{r}')
 Y_k^h(\mathbf{r})^*
 \end{equation}
 
The FKT leverages modern computational tools to build analogous low-rank approximations for a broad class of kernels.

\subsection{The Generalized Multipole Expansion}\label{sec:expansion}
We build our new technique by developing an expansion for general kernels into separable radial and angular functions as in \eqref{eq:electro}. We begin by defining $\varepsilon \coloneqq \frac{r'}{r}\left(\frac{r'}{r} - 2\cos{\gamma}\right)$,
where $\gamma$ is again the angle between $\mathbf{r'}$ and $\mathbf{r}$. Then
$ K(|\mathbf{r}'-\mathbf{r}|) = K(r\sqrt{1+\varepsilon})$ by the law of cosines,
and, assuming $r>0$ and $K$ is analytic except possibly at the origin, we can form a Taylor expansion around $\varepsilon=0$
\begin{equation}\label{eq:taylor}
K(|\mathbf{r'}-\mathbf{r}|) = 
\sum_{n=0}^{\infty} \frac{\varepsilon^n}{n!} \frac{\partial^n}{\partial \varepsilon^n}K(r\sqrt{1+\varepsilon})_{\varepsilon=0}
.
\end{equation}
By expanding the $\varepsilon^n$ terms via the binomial theorem, transforming from powers of cosine into Gegenbauer polynomials of cosine (via an identity from \citet{averygegen}, given in the appendix in~\eqref{eq:coschange}), and using Faa di Bruno's theorem for the derivatives with respect to $\varepsilon$, this sum can be rewritten as an expansion in (hyper)spherical harmonics, as given by Theorem~\ref{thm:main}
\begin{theorem}
If $K$ is analytic except possibly the origin, then for $\mathbf{r}'$,$\mathbf{r}$ within the radius of convergence, 
\[
K(|\mathbf{r'}-\mathbf{r}|)
=
\sum_{k=0}^\infty
\sum_{h\in\mathcal{H}_k}
Y_k^h(\mathbf{r})Y_k^h(\mathbf{r}')^* \mathcal{K}^{(k)}(r',r)\] 
where 
\begin{equation}\label{eq:kdef}
\mathcal{K}^{(k)}(r',r)\coloneqq
 \sum_{j=k}^{\infty}
r'^j
\sum_{m=1}^{j}
K^{(m)}(r)
r^{m-j}\mathcal{T}_{jkm}^{(\alpha)},
 \end{equation}
and $\mathcal{T}_{jkm}^{(\alpha)}$ are constants which depend only on the dimension and not on the kernel or data. The radius of convergence is the same as that of~\eqref{eq:taylor}.
\label{thm:main}
\end{theorem}
(See Section~\ref{sec:derivation} for the proof and the definition of $\mathcal{T}_{jkm}^{(\alpha)}$). We thus arrive at the approximation underlying the Fast Kernel transform, a truncated expansion with truncation parameter $p$.
\begin{equation}\label{eq:finaltrunc}
K(|\textbf{r}'-\textbf{r}|) \approx
\sum_{k=0}^{p}
\sum_{h\in \mathcal{H}_k} 
Y_k^h(\mathbf{r})
 Y_k^h(\mathbf{r}')^* \mathcal{K}^{(k)}_p(r',r)
 \end{equation}
where $\mathcal{K}^{(k)}_p$ is the $p$-term truncation of the infinite sum in the definition of $\mathcal{K}^{(k)}$. This expansion represents the kernel as a sum of products of functions of $\textbf{r}$ with functions of $\textbf{r}'$, which is the form called for by~\eqref{eq:lowrank}. We may use this expansion to generate the $s2m$ and $m2t$ matrices in Algorithm~\ref{alg:fact} by collecting the functions of $\mathbf{r}'$ into the $s2m$ matrix and the functions of $\mathbf{r}$ into the $m2t$ matrix. 
 
 The sums over $j$ and $k$ in the definition of $\mathcal{K}^{(k)}$ turns out to have interesting and helpful properties for our algorithm. In particular, for certain types of kernels it is possible to automatically compute more concise expansions than the form given in~\eqref{eq:kdef}, resulting in better complexity. The details of this additional compression can be found in Section~\ref{sec:rankradial}. 
\section{Analysis}\label{sec:analysis}
\subsection{Truncation Error}
The truncation~\eqref{eq:finaltrunc} yields error
$|\mathcal{E}_P|$, which we bound using Lemma~\ref{lem:error}.
\begin{lemma}[Truncation Error]
\begin{equation}\label{eq:truncerr}
\begin{aligned}
|\mathcal{E}_P|
\leq 
\sum_{k=0}^\infty  
\binom{k+d-3}{k}
\left|\sum_{j=\max{(p+1,k)}}^{\infty}
\sum_{m=1}^{j}
K^{(m)}(r)
r^{m}\left(\frac{r'}{r}\right)^j\mathcal{T}_{jkm}^{(\alpha)}
 \right|
\end{aligned}
\end{equation}
\label{lem:error}
\end{lemma}
\begin{proof}
This follows from the bound $|C_k^{(\alpha)}(\cos{\gamma})|\leq\binom{k+d-3}{k}$ on Gegenbauer polynomials \citep{gegen} and bringing the absolute value inside the sum.
\end{proof}

In Figure~\ref{fig:synthetic}, right, we report several empirical findings on this bound. As in the error analysis of the FMM for the electrostatic and Helmholtz kernels, the error is observed to decay exponentially with the choice of truncation parameter.  In practice, the above bound turns out to be fairly loose\textemdash as we report in Section~\ref{sec:experiments}, a choice of $p=4$ yields a residual less than $10^{-4}$ for reasonable distance criteria. Because the bound in Lemma~\ref{lem:error} is observed to be fairly loose (albeit descriptive) in practice, we omit further analysis. It is of interest to further develop and analyze tighter upper bounds.

\subsection{Computational Complexity}\label{sec:complex}
To assess the computational complexity of the FKT, we need to understand the size of our compressed far-field expansion. Our low rank approximation takes the form 
\[
K(|\mathbf{r}_i-\mathbf{r}_j|) \approx \sum_{k=0}^\mathcal{P} U_k(\mathbf{r}_i)V_k(\mathbf{r}_j)=
\sum_{k=0}^{p}
\sum_{h\in \mathcal{H}_k} 
Y_k^h(\mathbf{r})
 Y_k^h(\mathbf{r}')^* \mathcal{K}^{(k)}_p(r',r),
\]
and it is not hard to show (see Section~\ref{sec:numterms}) that 
$\mathcal{P}=\binom{p+d}{d}\sim d^p$.
We note that this is exactly the same as the number of terms in the expansion underlying the Improved FGT (\citep{yang2003improved}), and is achieved for a much broader class of kernels by the FKT. 


\begin{table}
\caption{Commonly used covariance functions in Gaussian process regression}
\begin{center}
\begin{small}
\begin{tabular}{ | c| c |  } 
\hline
 Exponential  & $K(r)=e^{-r}$ \\ \hline
 Mat\'ern $(\nu=3/2)$ & $\sigma^2(1+\frac{\sqrt{3}r}{\rho})e^{-\sqrt{3}r/\rho}$ \\  \hline
 Cauchy  & $\frac{1}{1+r^2/\sigma^2}$ \\  \hline
 Rational Quadratic ($\alpha=1/2$) & $\frac{1}{\sqrt{1+r^2/\sigma^2}}$ \\
 \hline
\end{tabular}
\end{small}
\end{center}
\label{table:cov_funs}
\end{table}

The complexity of Algorithm~\ref{alg:fact} is the sum of the cost of computing the dense matrices for nearby interactions, the cost of computing the $s2m$ matrices for every node, and the cost of computing the $m2t$ matrices \john{for} every node. For simplicity of this analysis, we assume that every leaf has at most $m$ points, each leaf has at most $N_d$ points in its near field, and each point is in the far field of at most $F_d$ nodes. If the total number of points is $N$, the total cost is given by
\begin{equation}\label{eq:firstcomplexity}
\text{FKT}_{cost} = \mathcal{O}\left(\frac{N}{m}mN_d + N\john{\log{(N/m)}}\mathcal{P} + NF_d\mathcal{P}\right)=\mathcal{O}\left(N\left(N_d\john{+\log{(N/m)}d^p}+F_dd^p\right)\right).\end{equation}
In practice, $F_d$ can be made to depend on the intrinsic\footnote{Data which approximately lies on a lower-dimensional manifold has \emph{intrinsic dimension} equal to that of the manifold. The \emph{ambient dimension} is the dimension of the space in which the data is expressed (e.g. a circle has intrinsic dimension 1 and ambient dimension 2). } dimension of the data by the choice of tree decomposition, but is generally exponential in that intrinsic dimension and has an additional factor of $\log{(N/m)}$ coming from the depth of the tree. $N_d$ depends on the maximum leaf capacity $m$ and a factor exponential in the intrinsic dimension. Letting $d_i, d_a$ be the intrinsic and ambient dimensions of the data, we have 
\begin{equation}\label{eq:secondcomplexity}
\text{FKT}_{cost} =\mathcal{O}\left(N\left(mc_n^{d_i}+\john{(1+}c_f^{d_i}\john{)}\log(N/m)d_a^p\right)\right)=
\mathcal{O}\left(N\log(N/d_a^p)\times c_f^{d_i}\times d_a^p\right)
\end{equation}
where we have set $m=\mathcal{O}(d_a^p)$ and $c_n,c_f$ are constants which depend on the problem geometry, typically between 2 and 5.
Note that in cases where the additional compression described at the end of Section~\ref{sec:expansion} is applied, the size $\mathcal{P}$ of the expansion can be reduced by a factor of $d$ and the $d_a^p$ term in \eqref{eq:secondcomplexity} becomes $d_a^{p-1}$.



\begin{figure}
\centering
   \includegraphics[width=0.42\columnwidth]{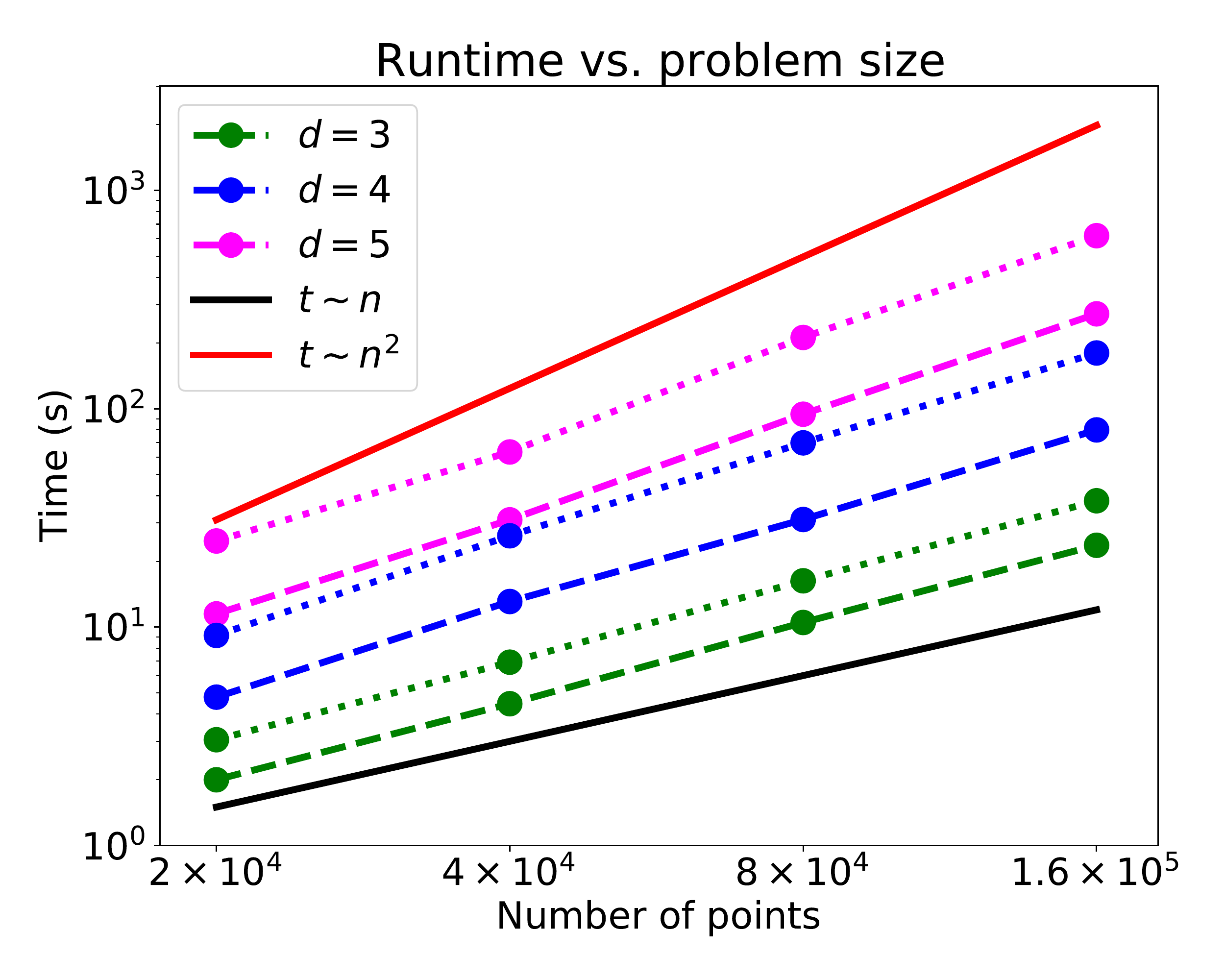}
   \hspace{5mm}
   \includegraphics[width=0.42\columnwidth]
{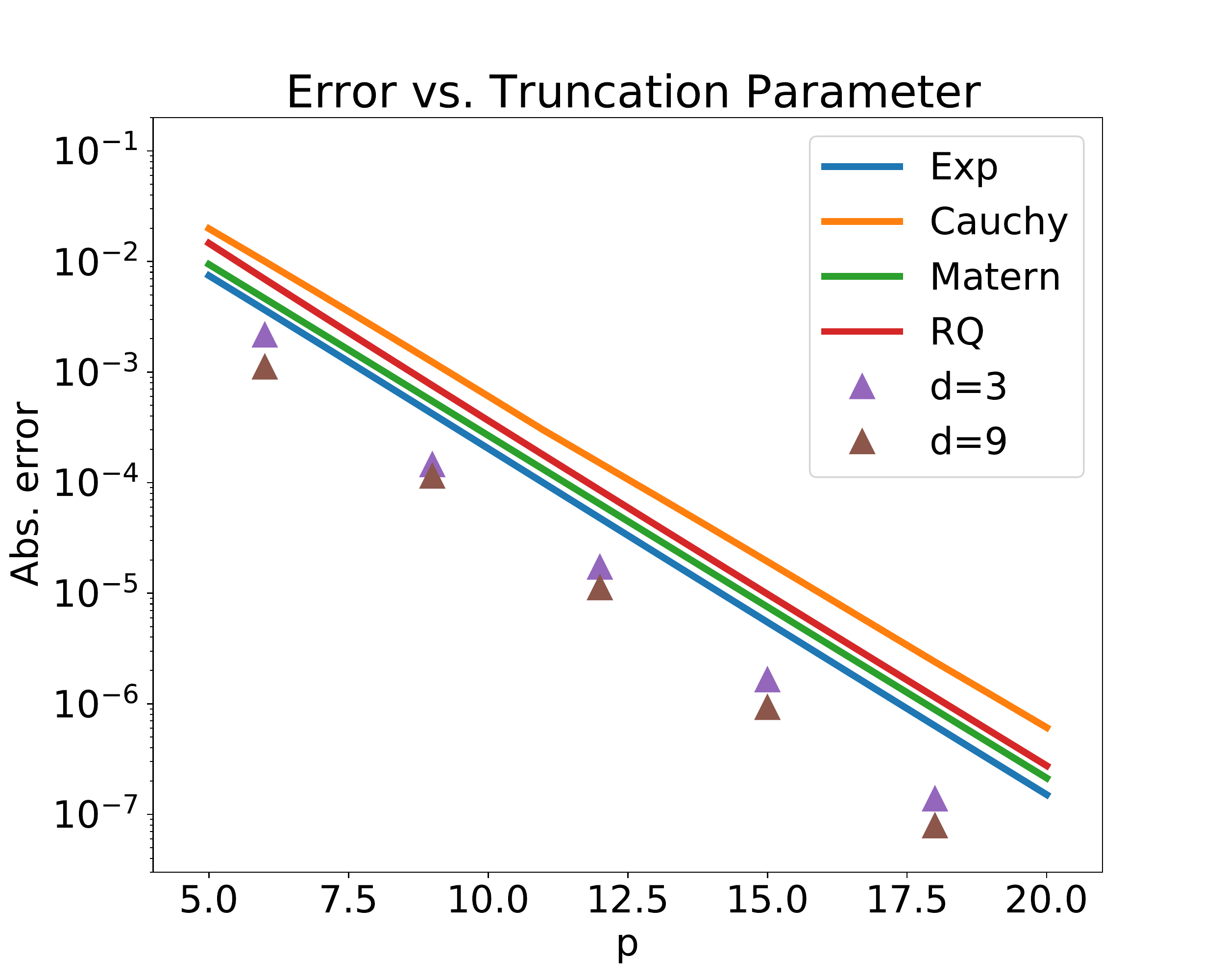}
\caption{Left: Runtimes of the FKT for matrix-vector multiplies for the Mat\'ern kernel with $\nu=1/2$. Dashed lines show results for $p=4$ and dotted lines show $p=6$.
Right: The lines are estimates of the upper bound for $d=3$ in \eqref{eq:truncerr} for various kernels found by fixing $r'/r=1/2$, summing from $p+1$ to $30$, and taking the maximum over 2000 uniformly chosen points in $r\in[0,20]$ (we do not see growth in $r$).
These estimates are shown for the Exponential, Mat\'ern, Cauchy, and Rational Quadratic kernels as described in Table~\ref{table:cov_funs}.
The triangles are experimentally observed errors which are calculated for the $p=4$ Cauchy kernel FKT approximation by taking the maximum absolute error of the truncated expansion for 1000 randomly selected pairs of points $\textbf{r}',\textbf{r}$ satisfying $|\textbf{r}'|=1,|\textbf{r}|=2$. }
\label{fig:synthetic}
\end{figure}

\subsection{Limitations}\label{sec:limitations}
The FKT will generally not scale well to datasets in high dimensions, although its underlying expansions remain accurate. The problem is that the method requires dense computation of points nearby each other, which leads to poor scaling in high dimensions when points tend to be closer together. In contrast, the FGT provide a low-rank approximation for points nearby to each other based on the global low-rankness of the Gaussian kernel. Although the FKT can provide low-rank approximations for distant points, it cannot yet do so for nearby points.  

Although the FKT automatically finds the analytical expansions foundational to the FMM, it scales quasi-linearly rather than linearly as the FMM does. One way to make the FKT a linear algorithm (taking further inspiration from the FMM) would be to develop translation operators for the expansion general to any kernel.

Finally, in contrast to the FGT and the FMM, the FKT lacks particularly helpful theoretical bounds on the error, owing mainly to its dense theoretical underpinning. We present empirical observations in this work, but future developments should provide deeper illumination into the error guarantees that can be given for kernels with known bounds on their derivatives. 

\section{Experiments}
\label{sec:experiments}
We've implemented the FKT in Julia as part of an open source toolkit\footnote{https://github.com/jpryan1/FastKernelTransform.jl}, making use of the NearestNeighbors.jl package \citep{nearestneighbor} to compute near and far sets of points, and the 
TaylorSeries.jl package \citep{benet2019taylor} to automatically compute derivatives. Both packages are licensed under the MIT ``Expat'' License. The synthetic experiments were performed single-threaded on a 2020 Apple Macbook Air with an M1 CPU and 8GB of RAM, and the regression experiment was performed on a 2017 MacBook with a Dual-Core Intel Core i7 and 16GB of RAM.
\subsection{Synthetic Data}



To test the runtime of the algorithm, we generate a synthetic dataset of points uniformly distributed on a unit hypersphere. We then approximate a matrix-vector multiplication with a Mat\'ern kernel matrix (see Table~\ref{table:cov_funs}) on this dataset against a random vector. Our test uses an distance parameter value of $\theta=0.75$, maximum leaf capacity of 512, and includes results for truncation parameter $p=4,6$. Results for this test in a variety of dimensions and problem sizes are shown in Figure~\ref{fig:synthetic}, left\textemdash the runtime is seen to be quasi-linear in the problem size. We observe the FKT to become faster than dense matrix multiplication at $N=1000$ for $d=3$, $N=5000$ for $d=4$, and $N=20,000$ for $d=5$. To test the accuracy of the approximation, we compare the truncated expansion to the true kernel value for the Cauchy kernel in 3 and 9 dimensions. The errors are calculated for the $p$-term approximation for 1000 randomly selected pairs of points $\textbf{r}',\textbf{r}$ satisfying $|\textbf{r}'|=1,|\textbf{r}|=2$, and $p$ is swept from $6$ to $18$ (see Figure~\ref{fig:synthetic}, right). The error is seen to decay exponentially with $p$ in both kernels, and not be affected by dimension. Results for this experiment in more dimensions and for more kernels can be found in Section~\ref{sec:moreerror}.

\subsection{Stochastic Neighborhood Embedding}

\begin{figure}
    \centering
    \hspace*{\fill}
    \includegraphics[width=0.42\columnwidth]{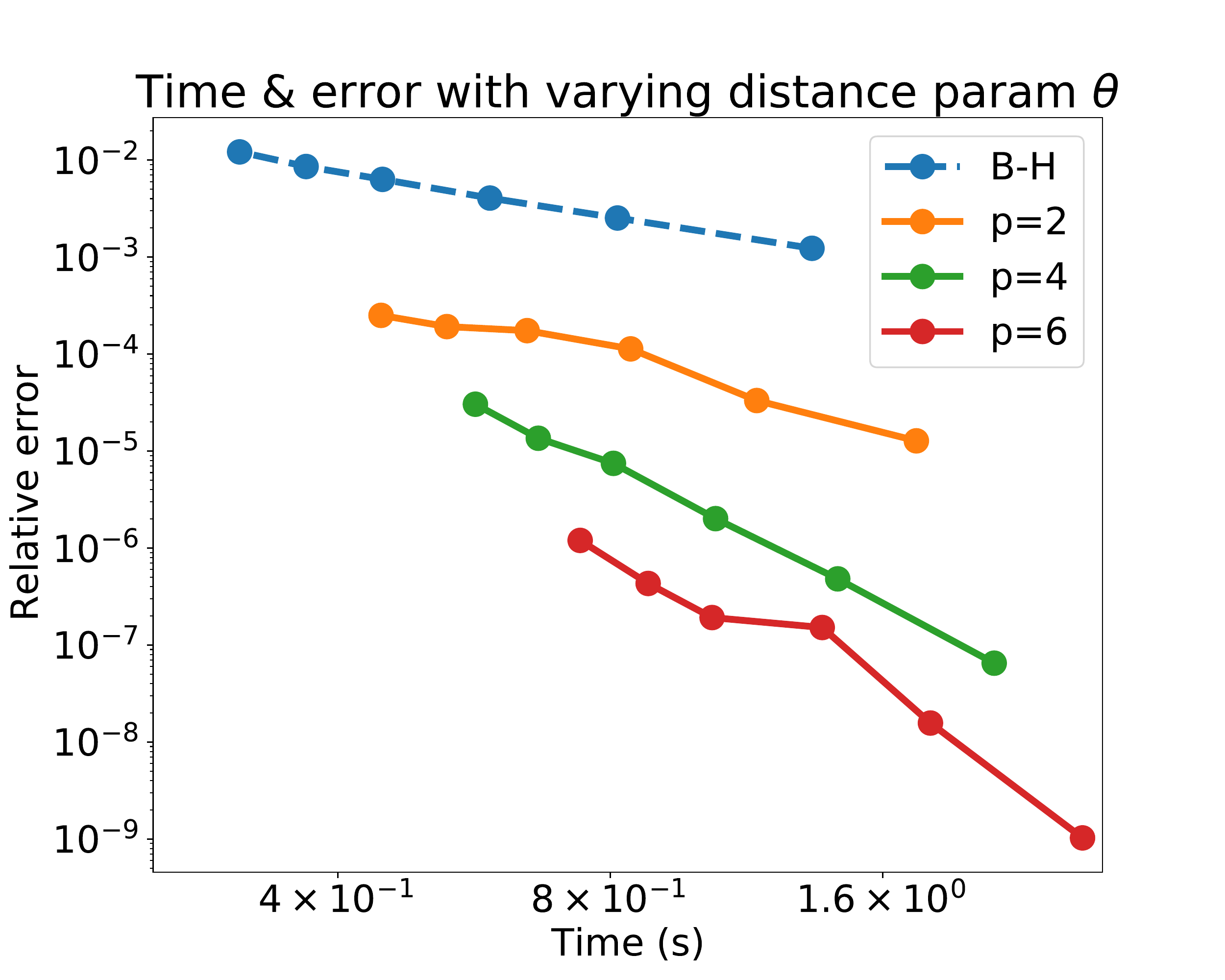}
    \hspace{3mm}
    \includegraphics[width=0.42\columnwidth]{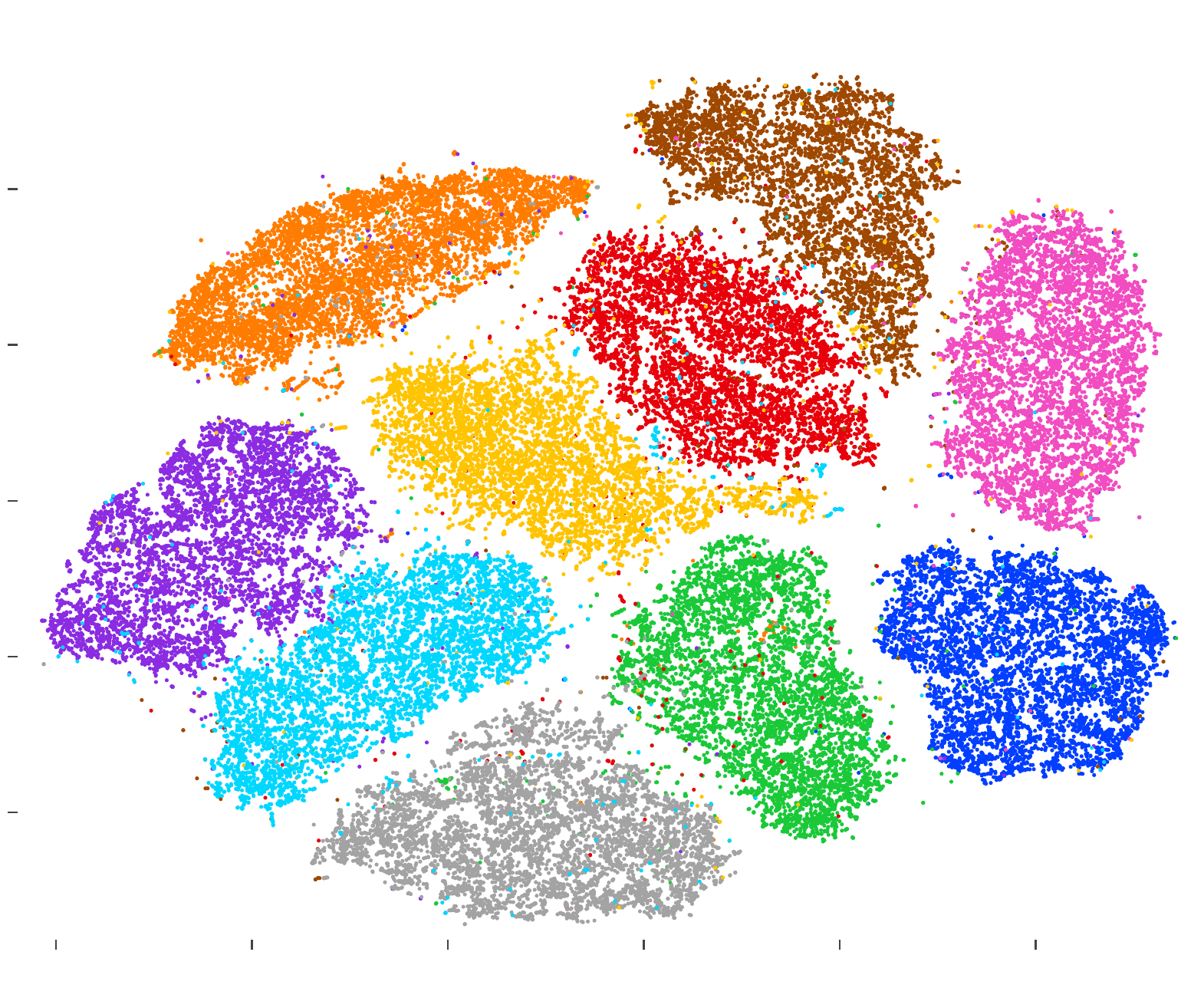}
    \hspace*{\fill}
    \caption{Left: Runtimes and relative errors for a series of matrix vector multiplies using the Cauchy kernel on a 2D dataset of 20k uniformly distributed points in the unit square. B-H refers to the Barnes-Hut method, which is equivalent to the $p=0$ FKT with centers of masses as the expansion centers. The maximum leaf capacity was 512, and for each $p$, we varied the distance parameter $\theta$ between 0.25 and 0.75. 
     Right: A t-SNE embedding of the MNIST training set of 60,000 images of digits computed via application of FKT.}
    \label{fig:tsne}
\end{figure}

The stochastic neighborhood embedding (SNE) was proposed by \cite{hinton2002stochastic},
and \cite{maaten2008visualizing} followed-up that work with the improved t-distributed SNE (t-SNE).
The t-SNE is widely used as an effective tool for dimensionality reduction for data visualization.
An implementation of its optimization routine requires sums of and matrix-vector-products with kernel matrices with $N^2$ entries, which does not scale well to large data sets.
In particular, the relevant gradient of the t-SNE objective contains matrix-vector products with a kernel matrix of the Cauchy kernel $(1+r^2)^{-1}$ with two-dimensional inputs,
which is a prime candidate for the application of FKT.
Previously, \cite{van2014accelerating} proposed accelerated methods for t-SNE based on tree codes including the aforementioned Barnes-Hut scheme. While the Barnes-Hut scheme is simpler,
Fig.~\ref{fig:tsne}, left shows that FKT exhibits a superior accuracy-runtime trade-off if more accuracy is desired.
The plot was generated by varying the $\theta$ distance parameter in a similar vein as \cite{van2014accelerating}.
While very high accuracy might not be of utmost concern for the optimization of a t-SNE, which has a more qualitative goal,
this result more generally demonstrates that FKT achieves a very high degree of accuracy, while also highlighting FKT's generality, since it needs {\it no manual adaption} to work on the relevant matrices
and compute the visualization of MNIST \citep{mnist} in Fig. \ref{fig:tsne}. The MNIST data is licensed under the CC BY-SA 3.0 license.

\subsection{Gaussian Process Regression of Sea Surface Temperature}

\begin{figure}[t!]
\centering 
\includegraphics[width=1\columnwidth]
{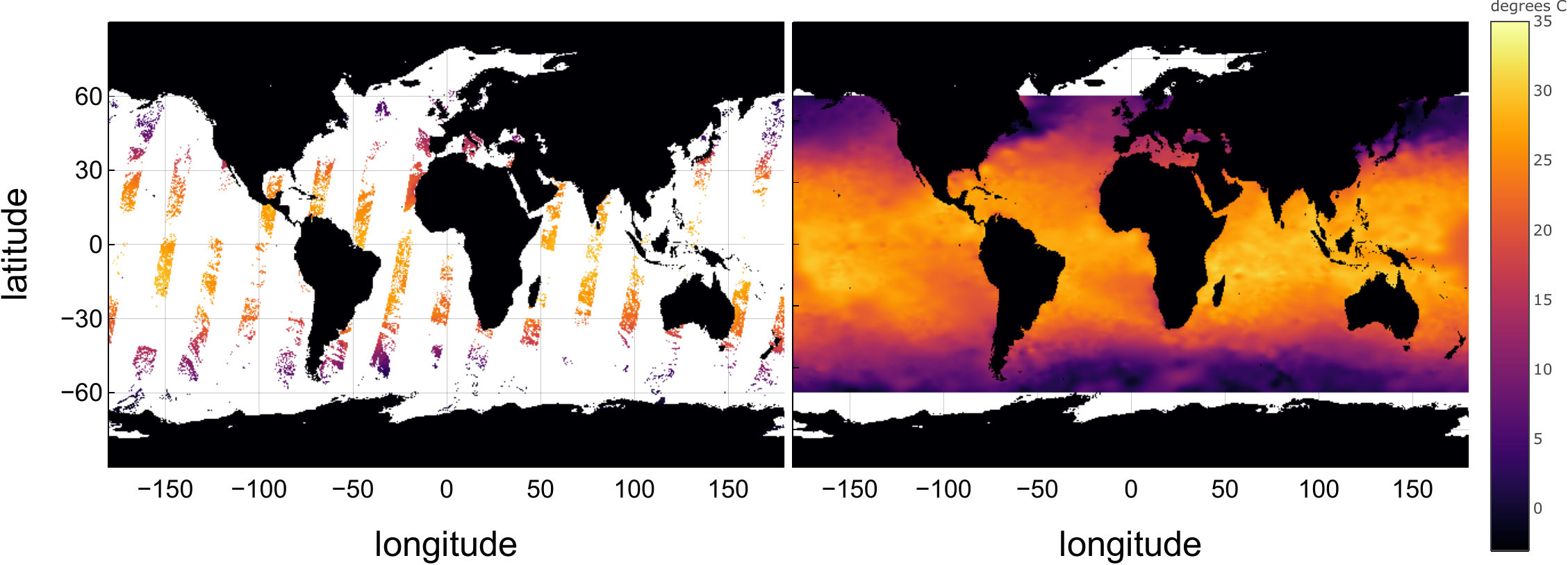}
\caption{
Sea surface temperatures collected by a single satellite throughout one day (left) and
the posterior mean of a Gaussian process with a Mat\'ern kernel conditioned on seven days of data.
}
\vskip -0.2in
\label{fig:SST_measurement}
\end{figure}

Gaussian processes (GPs) constitute another important class of a kernel methods.
Importantly, inference of the posterior predictive mean of a GP can be carried out exclusively through matrix-vector multiplications with kernel matrices and a diagonal ``noise'' variance matrix \citep{wang2019exact}.
To highlight the generality of FKT, we use it here to compute a GP regressor on sea surface temperature data from Copernicus, the European Union's Earth Observation Programme \citep{merchant2019satellite}, which is licensed under the CC-BY 4.0 license.
The data set is collected by a satellite orbiting the earth several times per day, leading to measurement locations with a complex spatial structure (see Fig. \ref{fig:SST_measurement}, left).
Each data point comes with an uncertainty estimate, which we use to populate the diagonal noise variance matrix 
of the model.
We consider data for the first seven days of 2019, for which more than 8 million data points were collected
and sub-sampled it down to a still considerable 145,913 observations
by taking every 56th data point in the temporal order in which they were collected. 
We then evaluated the posterior predictive mean of a GP with the Mat\'ern-3/2 kernel conditioned on the observations and their uncertainties at 480,000 predictive points to arrive at the result on the right of Fig. \ref{fig:SST_measurement}.
We restricted the predictions to be within 60 degrees of latitude of the equator, since the satellite data is very sparse in the polar regions.
The entire computation completed in around twelve minutes on a 2017 MacBook with a Dual-Core Intel Core i7 and 16GB of RAM, highlighting once more the rare combination of generality and high efficiency that FKT achieves.



\section{Discussion and broader impacts}
We've presented the Fast Kernel Transform, a general method for the automatic computation of analytical expansions of isotropic kernels which can be used in hierarchical matrix algorithms on datasets in moderate dimensions. 
The FKT has a high, quantifiable, and controllable level of accuracy, and its cost grows only quasi-linearly in the number of data points and polynomially in the ambient dimension. While our work is entirely algorithmic in nature, it is important to remark that using approximation schemes such as the FKT can introduce additional variation in downstream tasks that are not anticipated. While we provide controllable levels of accuracy, it is still important to assess the level of sensitivity of different applications to such perturbations and validate models developed with these methods across a broad range of criteria.

The method develops a new analytic approximation scheme whose number of terms is equal to those of the expansions developed for the Improved FGT, but for a much broader set of kernels. 
At its core, our method reflects a generalization of the mathematical tools underlying seminal works in kernel methods, such as the FMM and the FGT, and opens up many opportunities for further theoretical study and algorithmic development, such as work on a more rigorous foundation for the class of kernels for which the FKT excels,
and work on removing the ambient dimension from the cost of FKT via an appropriate selection of harmonics to retain when an underlying intrinsically lower-dimensional manifold is known or may be discovered.
Further, the logarithmic term could be removed by the creation of translation operators so as to completely generalize the FMM to this broad class of kernels. These translation operators are the subject of current development by the authors.
We believe that the methods contained herein could prove useful for a wide range of practitioners and researchers of kernel methods, enabling them to apply their methods to much larger problem instances than without acceleration,
and have made an open-source implementation of FKT available.



\bibliography{fkt}
\bibliographystyle{icml2021}

\appendix 
\section{Technical Details}
In this section, we lay out the derivation of the expansion underlying the Fast Kernel transform. Before the derivation, we review the Gegenbauer polynomials which will feature heavily. Additionally, we expand on the opportunity for additional compression for certain types of kernels alluded to in the main text. Finally we will show the details of the computation of the number of terms in the FKT expansion.

\subsection{Gegenbauer Polynomials}
\label{sec:gegenbauer}
The generalized multipole expansion is expressed in terms of Gegenbauer polynomials, also known as ultraspherical polynomials \citep{askey1983generalization}.
For our purposes, these polynomials are best seen as generalizations of the Legendre polynomials which have higher dimensional addition theorems. 
They satisfy the recurrence relation
\begin{equation}
\begin{aligned}
    C^\alpha_0(x) &= 1, \\ 
    C^\alpha_1(x) &= 2\alpha x, \\ 
    C^\alpha_n(x) &= \big[ 2x(n+\alpha-1) C^\alpha_{n-1}(x)- (n+2\alpha-2) C^\alpha_{n-2}(x) \big] / n,
\end{aligned}
\end{equation}
and the hyperspherical harmonic addition theorem \citep{averyproperties}
 \begin{equation}\label{eq:additiontheorem}
 \frac{1}{Z_k^{(\alpha)}}C_k^{(\alpha)}(\cos{\gamma})
 = \sum_{h\in \mathcal{H}_k}
 Y_k^h(\mathbf{r}')
 Y_k^h(\mathbf{r})^*,
 \end{equation}
 where $\mathbf{r},\mathbf{r'}\in\mathbb{R}^d$ have angle $\gamma$ between them,  $\alpha=\frac{d}{2}-1$, $Z_k^{(\alpha)}$ is a normalization term, and  \[
 \mathcal{H}_k\coloneqq \{(\mu_1, \dots \mu_{d-2}) : k\geq\mu_1 \geq \dots \geq |\mu_{d-2}|\geq 0\}.
 \]

\subsection{Derivation of the FKT expansion}\label{sec:derivation}
Before going through the proof of the main theorem of the main text, we will need a lemma concerning the application of Faa di Bruno's theorem to our particular composition of functions ($f(g(\varepsilon))$ where $g(\varepsilon) = r\sqrt{1+\varepsilon}$ and $f(g(\varepsilon))=K(r\sqrt{1+\varepsilon})$.  Before \emph{that} lemma, we prove a combinatorial identity which will be necessary. 
\begin{lemma}\label{lm:jensen}
\[
\sum_{k=0}^n \binom{m+1}{k}=
\sum_{k=0}^n \binom{2k+1}{k}\binom{m-(2k+1)}{n-k}.
\]
\end{lemma}
\begin{proof}
As a preliminary, note that the LHS are entries in Bernoulli's triangle, and hence satisfy
\[
b_{m,n} = 
\begin{cases}
2^m-1,& m=n \\
b_{m-1,n} + b_{m-1, n-1},& m>n>0\\
1,& n=0.
\end{cases}
\]
It will suffice to show that the RHS follows the same recurrence. 
We refer to the following result from Jenson, 1902:
\[\sum_{k=0}^n \binom{2k+1}{k}\binom{m-(2k+1)}{n-k}  
=
\sum_{k=0}^n
\binom{m-k}{n-k}2^k.
\]
By inspection the $m=n$ and $n=0$ cases are immediately confirmed. If $m>n>0$ then
\[
\sum_{k=0}^n
\binom{m-k-1}{n-k}2^k
 + 
 \sum_{k=0}^{n-1}
\binom{m-k-1}{n-k-1}2^k
\]

\[ = 2^n + \sum_{k=0}^{n-1}
\left(
\binom{m-k-1}{n-k} + \binom{m-k-1}{n-k-1}
\right)2^k\]
\[ = 2^n + \sum_{k=0}^{n-1}
\binom{m-k}{n-k}
2^k = \sum_{k=0}^{n}
\binom{m-k}{n-k}2^k.\]

\end{proof}

\begin{lemma}\label{lm:bell}
When $n>0$, 
\begin{equation}\label{eq:faa}
\frac{\partial^n}{\partial \varepsilon^n}\left(K(r\sqrt{1+\varepsilon})\right)|_{\varepsilon=0} 
= 
\sum_{m=1}^n \mathcal{B}_{nm}K^{(m)}(r)r^m,
\end{equation}
where 
\[\mathcal{B}_{nm}=(-1)^{n+m}\frac{(2n-2m-1)!!}{2^n}\binom{2n-m-1}{m-1},\]
and we will use the notation $K^{(m)}(r)$ to mean $\frac{\partial^m}{\partial r^m}K(r).$
\end{lemma}
\begin{proof}
Let $g(\varepsilon)=r\sqrt{1+\varepsilon}$ and note that 
\begin{equation}
g^{(i)}(\varepsilon)_{\varepsilon=0} = (-1)^{i+1}\frac{(2i-3)!!}{2^i}r,
\end{equation}
where we will let $(2i-3)!!=1$ when $i=1$. 
By \citet{Riordan}
\[
\frac{\partial^n}{\partial \varepsilon^n}\left(K(r\sqrt{1+\varepsilon})\right)|_{\varepsilon=0} =
\sum_{m=1}^n K^{(m)}(g(\varepsilon)_{\varepsilon=0}) \cdot B_{n,m}(g'(\varepsilon)_{\varepsilon=0}, g''(\varepsilon)_{\varepsilon=0}, ..., g^{(n-m+1)}(\varepsilon)_{\varepsilon=0})
\]
where $B_{n,m}$ are the Bell polynomials (henceforth we will drop their arguments). Per usual, we set 
\[B_{0,0}=1\quad B_{n,0}=B_{0,m}=0.\]
Then the Bell polynomials satisfy the recurrence relation

\[
B_{n,m} = \sum_{i=1}^{n-m+1}\binom{n-1}{i-1}g^{(i)}(\varepsilon)_{\varepsilon=0}B_{n-i, m-1}.
\]
We will use this to prove 
\[
B_{n,m} = (-1)^{n+m}\frac{(2n-2m-1)!!}{2^n}\binom{2n-m-1}{m-1}r^m
\]\[ = 
(-1)^{n+m}r^m\frac{(2n-m-1)!}{(m-1)!(n-m)!2^{2n-m}}\quad n\geq m > 0,
\]
by induction. We begin with the base cases of $n=m=1$ and $n>m=1$. For the former, the recurrence relation yields
\[B_{1,1} = g'(\varepsilon)_{\varepsilon=0} = \frac{1}{2}r,\]
and our claim yields
\[B_{1,1} = (-1)^2r^1\frac{0!}{2^{1}0!0!} = \frac{1}{2}r.\]
When $n>m=1$ the recurrence relation gives

\[B_{n,1} = \sum_{i=1}^n\binom{n}{0}
g^{(i)}(\varepsilon)_{\varepsilon=0}B_{n-i,0} = g^{(n)}(\varepsilon)_{\varepsilon=0} = (-1)^{(n+1)}r\frac{(2n-3)!!}{2^n},\]
and our claim yields
\[
B_{n,1} = (-1)^{n+1}r^1\frac{(2n-2)!}{2^{(2n-1)}0!(n-1)!}
=
(-1)^{(n+1)}r\frac{(2n-3)!(2n-2)}{2^{(2n-2)}(n-2)!(2n-2)}
\]
\[
= 
(-1)^{(n+1)}r\frac{(2n-3)!}{2^{(2n-2)}(n-2)!} = 
(-1)^{(n+1)}r\frac{(2n-3)!!}{2^{n}}
\]
For the inductive step, we need to show that 
\[
B_{n,m} = \sum_{i=1}^{n-m+1}
\binom{n-1}{i-1}(-1)^{i+1}\frac{(2i-3)!!}{2^i}r \frac{(-1)^{n-i+m-1}r^{m-1}(2n-2i-m)!}{(m-2)!(n-i+m-1)!2^{2n-2i-m+1}}
\]
\[
=
r^m\frac{(-1)^{n+m}}{2^{2n-m+1}(m-2)!}\sum_{i=1}^{n-m+1}\binom{n-1}{i-1}2^i\frac{(2i-3)!!(2n-2i-m)!}{(n-i-m+1)!}.
\]
Separating the $i=1$ term out so that the double factorial is of positive integers
\[
 = 
 r^m\frac{(-1)^{n+m}}{2^{2n-m+1}(m-2)!}\left(
 \frac{2(2n-m-2)!}{(n-m)!} + \sum_{i=2}^{n-m+1}\binom{n-1}{i-1}2^i\frac{(2i-3)!!(2n-2i-m)!}{(n-i-m+1)!}\right).
\]

Moving some terms out and rewriting the double factorial
\[
 = 
 r^m\frac{(-1)^{n+m}(2n-m-1)!}{2^{2n-m}(m-1)!(n-m)!}\]\[\cdot \left(
\frac{m-1}{2n-m-1} + \frac{m-1}{(2n-m-1)!}
\sum_{i=2}^{n-m+1}
\binom{n-1}{i-1}2^{i-1}\frac{(2i-3)!(2n-2i-m)!(n-m)!}{2^{i-2}(i-2)!(n-i-m+1)!}\right).
\]
Evidently we are done if the large parenthetical is equal to 1, which is equivalent to 
\[
\frac{2n-m-1-(m-1)}{2n-m-1} = \frac{m-1}{(2n-m-1)!}
\sum_{i=2}^{n-m+1}
\binom{n-1}{i-1}2\frac{(2i-3)!(2n-2i-m)!(n-m)!}{(i-2)!(n-i-m+1)!}
\]
\[
(2n-m-2)! = (m-1) 
\sum_{i=2}^{n-m+1}
\binom{n-1}{i-1}\frac{(2i-3)!(2n-2i-m)!(n-m-1)!}{(i-2)!(n-i-m+1)!}.
\]
Starting the sum at $i=1$,
\[
(2n-m-2)! =  (m-1)
\sum_{i=1}^{n-m}
\binom{n-1}{i}\frac{(2i-1)!(2n-2i-m-2)!(n-m-1)!}{(i-1)!(n-i-m)!}.
\]
Now we break the binomial coefficient into factorials and rearrange into new binomial coefficients
\[
(2n-m-2)! =  (m-1)
\sum_{i=1}^{n-m}
\frac{(n-1)!(2i-1)!(2n-2i-m-2)!(n-m-1)!}{(n-1-i)!(i)!(i-1)!(n-i-m)!}
\]

\[
(2n-m-2)! =  (m-1)
\sum_{i=1}^{n-m}
\binom{2i-1}{i} 
\binom{2n-2i-m-1}{n-i-1}
\frac{(n-1)!(n-m-1)!}{(2n-2i-m-1)}.
\]
Moving $m-1$ into the sum and some factorials to the LHS
\[
\binom{2n-m-2}{n-m-1} =
\sum_{i=1}^{n-m}
\binom{2i-1}{i} 
\binom{2n-2i-m-1}{n-i-1}
\frac{m-1}{2n-2i-m-1}
\]

\[
= \sum_{i=1}^{n-m}
\binom{2i-1}{i} 
\binom{2n-2i-m-1}{n-i-1}
\left(1 - \frac{2(n-i-m)}{2n-2i-m-1}\right)
\]
\[
= \sum_{i=1}^{n-m}
\binom{2i-1}{i} 
\binom{2n-2i-m-1}{n-i-1}
- 2
\sum_{i=1}^{n-m-1}
\binom{2i-1}{i} 
\binom{2n-2i-m-2}{n-i-2}.
\]
Note that the second sum goes to $n-m-1$ since the $i=n-m$ term gave zero. We set the sum variable to start at zero
\[
= \sum_{i=0}^{n-m-1}
\binom{2i+1}{i} 
\binom{2n-m-2-(2i+1)}{n-i}
- 2
\sum_{i=0}^{n-m-2}
\binom{2i+1}{i} 
\binom{2n-m-(2i+1)-1}{n-i-1}.
\]
Applying Lemma \ref{lm:jensen} to both sums yields
\[
\binom{2n-m-2}{n-m-1} =
\sum_{i=0}^{n-m-1}\binom{2n-m-1}{i} - 2\sum_{i=0}^{n-m-2}\binom{2n-m-2}{i}. 
\]
Applying Pascal's identity to the first sum and then combining the two into a telescoping sum yields the desired result.
\end{proof}
We now move to the derivation of the FKT's expansion.
 In short, the derivation proceeds by Taylor expanding in a variable which is small for well-separated points, rearranging into a Gegenbauer expansion, and replacing the derivative term with the simpler form via the above lemma. 
\begin{theorem}
If $K$ is analytic except possibly the origin, then for $\mathbf{r}'$,$\mathbf{r}$ within the radius of convergence, 
\[
K(|\mathbf{r'}-\mathbf{r}|)
=
\sum_{k=0}^\infty
\sum_{h\in\mathcal{H}_k}
Y_k^h(\mathbf{r})Y_k^h(\mathbf{r}')^* \mathcal{K}^{(k)}(r',r),\] 
where 
\[
\mathcal{K}^{(k)}(r',r)\coloneqq
 \sum_{j=k}^{\infty}
r'^j
\sum_{m=1}^{j}
K^{(m)}(r)
r^{m-j}\mathcal{T}_{jkm}^{(\alpha)},
\]
and $\mathcal{T}_{jkm}^{(\alpha)}$ are constants which depend only on the dimension and not on the kernel or data. The radius of convergence is the same as that of~\ref{eq:taylor}.
\end{theorem}
\begin{proof}
\[K(|\textbf{r}'-\textbf{r}|)=K(r\sqrt{1+\varepsilon}).\]
Taylor expanding around $\varepsilon=0$,
\[=\sum_{n=0}^\infty \varepsilon^n \frac{1}{n!} \frac{\partial^n}{\partial \varepsilon^n} \left(K(r\sqrt{1+\varepsilon})\right)_{\varepsilon=0}. \]
Noting $\varepsilon = \left(\frac{r'^2}{r^2}-2\frac{r'}{r}\cos{\gamma}\right)$ and expanding the binomial

\begin{equation}\label{eq:prelemma}
=\sum_{n=0}^\infty
\sum_{i=0}^n
\frac{r'^{2(n-i)}}{r^{2(n-i)}}
\left(-2\frac{r'}{r}\cos{\gamma}\right)^i
\binom{n}{i}
\frac{1}{n!}
\frac{\partial^n}{\partial \varepsilon^n} \left(K(r\sqrt{1+\varepsilon})\right)_{\varepsilon=0}.
\end{equation}

We will make use of \citet{averygegen}
\begin{equation}\label{eq:coschange}
\cos^i{\gamma} = \sum_{k=0}^i\mathcal{A}_{ki}C_k^{\alpha}(\cos{\gamma}),
\end{equation}
where $\alpha=d/2-1$, $C_k^{\alpha}(\cos{\gamma})$ is the Gegenbauer polynomial, $\mathcal{A}_{ki}=0$ when $k\neq i \mod 2$, and
\begin{equation}
\mathcal{A}_{ki} = \frac{i!(\alpha+k)}{2^i
\frac{i-k}{2}
(\alpha)_{ \frac{i+k}{2} + 1 }}
\end{equation}
when $k= i \mod 2$. Here $(\alpha)_{ \frac{i+k}{2} + 1 }$ denotes the rising factorial, i.e. $(\alpha)_n = (\alpha)(\alpha+1)\dots(\alpha+n-1)$. Then, substituting in for the powers of cosine in \eqref{eq:prelemma} yields
\[=\sum_{n=0}^\infty \sum_{i=0}^n \sum_{k=0}^i
\mathcal{A}_{ki}
C_k^{(\alpha)}(\cos{\gamma})
\frac{r'^{2(n-i)}}{r^{2(n-i)}}
\left(-2\frac{r'}{r}\right)^i
\binom{n}{i}
\frac{1}{n!}
\frac{\partial^n}{\partial \varepsilon^n} \left(K(r\sqrt{1+\varepsilon})\right)_{\varepsilon=0}.
\]
We pause to show that this triple sum is absolutely convergent. 
Let $\varepsilon_{\gamma=0}=\frac{r'^2}{r^2}$ be the value of $\varepsilon$ with $\gamma$ set to 0, and assume that this value is inside the radius of convergence of the above Taylor series in $\varepsilon$. Then
\[
\sum_{n=0}^\infty \sum_{i=0}^n \sum_{k=0}^i
\left|
\mathcal{A}_{ki}
C_k^{(\alpha)}(\cos{\gamma})
\frac{r'^{2(n-i)}}{r^{2(n-i)}}
\left(-2\frac{r'}{r}\right)^i
\binom{n}{i}
\frac{1}{n!}
\frac{\partial^n}{\partial \varepsilon^n} \left(K(r\sqrt{1+\varepsilon})\right)_{\varepsilon=0}\right|
\]
\[\leq \sum_{n=0}^\infty \sum_{i=0}^n
\left|
\frac{r'^{2(n-i)}}{r^{2(n-i)}}
\left(-2\frac{r'}{r}\right)^i
\binom{n}{i}
\frac{1}{n!}
\frac{\partial^n}{\partial \varepsilon^n} \left(K(r\sqrt{1+\varepsilon})\right)_{\varepsilon=0}\right|
 \sum_{k=0}^i
 \left|\mathcal{A}_{ki}
C_k^{(\alpha)}(1)
\right|
\]
\[= \sum_{n=0}^\infty\sum_{i=0}^n
\left|
\frac{r'^{2(n-i)}}{r^{2(n-i)}}
\left(-2\frac{r'}{r}\right)^i
\binom{n}{i}
\frac{1}{n!}
\frac{\partial^n}{\partial \varepsilon^n} \left(K(r\sqrt{1+\varepsilon})\right)_{\varepsilon=0}\right|
\]
\[=\sum_{n=0}^\infty
\left|\frac{\partial^n}{\partial \varepsilon^n} \left(K(r\sqrt{1+\varepsilon})\right)_{\varepsilon=0}
\frac{1}{n!}
\right|
\sum_{i=0}^n
\left|
\frac{r'^{2(n-i)}}{r^{2(n-i)}}
\left(-2\frac{r'}{r}\right)^i
\binom{n}{i}
\right|
\]\[=
\sum_{n=0}^\infty
\left|\frac{\partial^n}{\partial \varepsilon^n} \left(K(r\sqrt{1+\varepsilon})\right)_{\varepsilon=0}
\varepsilon_{(\gamma=0)}^n
\frac{1}{n!}
\right|.
\]
Since $|\varepsilon_{\gamma=0}|$ is inside the radius of convergence of the Taylor series, then the final sum above is finite as a consequence of the Taylor series being absolutely convergent in its radius of convergence. 

This absolute convergence allows us to swap the sums as we please, which we will do. First, let $j=2n-i$ so that $i=2n-j$, then
\[
\sum_{n=0}^\infty \sum_{i=0}^n \sum_{k=0}^i
=
\sum_{n=0}^\infty \sum_{j=n}^{2n} \sum_{k=0}^{2n-j}
=
 \sum_{j=0}^{\infty}
 \sum_{n=\lceil j/2\rceil}^j 
 \sum_{k=0}^{2n-j}
 =
 \sum_{j=0}^{\infty}
 \sum_{k=0}^j
 \sum_{n=\frac{j+k}{2}}^j
 =
 \sum_{k=0}^\infty
 \sum_{j=k}^\infty
 \sum_{n=\frac{j+k}{2}}^j.
\]
So our current form of the expansion is 
\[
\sum_{k=0}^\infty
 \sum_{j=k}^\infty
 \sum_{n=\frac{j+k}{2}}^j
\mathcal{A}_{k,(2n-j)}
C_k^{(\alpha)}(\cos{\gamma})
\frac{r'^{2(n-(2n-j))}}{r^{2(n-(2n-j))}}
\left(-2\frac{r'}{r}\right)^{(2n-j)}
\frac{1}{n!}
\frac{\partial^n}{\partial \varepsilon^n} \left(K(r\sqrt{1+\varepsilon})\right)_{\varepsilon=0}
\]
\begin{equation}\label{eq:prefaa}
=
\sum_{k=0}^\infty
C_k^{(\alpha)}(\cos{\gamma})
 \sum_{j=k}^\infty
 \sum_{n=\frac{j+k}{2}}^j
\mathcal{A}_{k,(2n-j)}
\frac{r'^{2(n-(2n-j))}}{r^{2(n-(2n-j))}}
\left(-2\frac{r'}{r}\right)^{(2n-j)}
\frac{1}{n!}
\frac{\partial^n}{\partial \varepsilon^n} \left(K(r\sqrt{1+\varepsilon})\right)_{\varepsilon=0}.
\end{equation}

The pieces are now in place for us to arrive at the FKT's final form. Plugging~\eqref{eq:faa} into~\eqref{eq:prefaa} yields
\[
\sum_{k=0}^\infty
C_k^{(\alpha)}(\cos{\gamma})
 \sum_{j=k}^\infty
 \sum_{n=\frac{j+k}{2}}^j
 \sum_{m=1}^n
\mathcal{A}_{k,(2n-j)}
\frac{r'^{2(n-(2n-j))}}{r^{2(n-(2n-j))}}
\left(-2\frac{r'}{r}\right)^{(2n-j)}
\frac{1}{n!}
\mathcal{B}_{nm}K^{(m)}(r)r^m
\]

\[
=\sum_{k=0}^\infty
C_k^{(\alpha)}(\cos{\gamma})
 \sum_{j=k}^\infty
 r'^j
 \sum_{m=1}^j
 K^{(m)}(r)r^{m-j}
\overline{\mathcal{T}}_{k,j,m}^{(\alpha)},
\]
where 
\[
\overline{\mathcal{T}}_{k,j,m}^{(\alpha)}\coloneqq
 \sum_{n=\max{(\frac{j+k}{2},m)}}^j
\mathcal{A}_{k,(2n-j)}
(-2)^{(2n-j)}
\frac{1}{n!}
\mathcal{B}_{nm}.
\]

Finally, expanding the Gegenbauer polynomial into hyperspherical harmonics yields
\[
=\sum_{k=0}^\infty
\sum_{h\in\mathcal{H}_k}
Y_k^h(\textbf{r})Y_k^h(\textbf{r}')^*
 \sum_{j=k}^\infty
 r'^j
 \sum_{m=1}^j
 K^{(m)}(r)r^{m-j}
\mathcal{T}_{k,j,m}^{(\alpha)},
\]
\end{proof}
where $\mathcal{T}_{k,j,m}^{(\alpha)}
=Z_k^{(\alpha)}\overline{\mathcal{T}}_{k,j,m}^{(\alpha)}$.

\subsection{Number of terms in FKT}\label{sec:numterms}
The ``rank'' of the low-rank expansion is given by
\begin{equation}\label{eq:num_multipoles}
\sum_{k=0}^P |\mathcal{H}_k|\lfloor \frac{P-k}{2}+1\rfloor,
\end{equation}
where $|\mathcal{H}_k|$ is the number of linearly independent hyperspherical harmonics of order $k$, and $\lfloor \frac{P-k}{2}+1\rfloor$ is the rank of $\mathcal{K}^{(k)}_p$. The former is given by 
$|\mathcal{H}_k| = \binom{k+d-1}{k}-\binom{k+d-3}{k-2}$ in \citet{averyproperties}.
We start by writing 
$\lfloor \frac{P-k}{2}+1\rfloor = \frac{P-k+1}{2}+\frac{1}{2}(1_{k=P\mod 2})$ and addressing the first term first. 
\[\sum_{k=0}^P
\left(\binom{k+d-1}{k}-\binom{k+d-3}{k-2}\right)
\frac{P-k+1}{2}
\]

\[=\sum_{k=0}^P
\binom{k+d-1}{k}\frac{P-k+1}{2}
-
\sum_{k=2}^P
\binom{k+d-3}{k-2}\frac{P-k+1}{2}.
\]
Further breaking apart the sum, 
\[=
\frac{1}{2}
P
\sum_{k=0}^P
\binom{k+d-1}{k}
-\frac{1}{2}
\sum_{k=0}^P
\binom{k+d-1}{k}(k-1)
\]\[
-\frac{1}{2}
P
\sum_{k=0}^{P-2}
\binom{k+d-1}{k}
+\frac{1}{2}
\sum_{k=0}^{P-2}
\binom{k+d-1}{k}(k+1).
\]
Applying the hockey stick identity yields
\[=
\frac{1}{2}
P
\binom{d+P}{P}
-\frac{1}{2}
\sum_{k=0}^P
\binom{k+d-1}{k}(k-1)
\]\[
-\frac{1}{2}
P
\binom{d+P-2}{P-2}
+\frac{1}{2}
\sum_{k=0}^{P-2}
\binom{k+d-1}{k}(k+1).
\]
Combining the two remaining sums
\[=
\frac{1}{2}
P
\binom{d+P}{P}
-\frac{1}{2}P\binom{d+P-2}{P-2}
\]\[
-\frac{1}{2}(P-2)\binom{d+P-2}{P-1}
-\frac{1}{2}(P-1)\binom{d+P-1}{P} + \sum_{k=0}^{P-2}\binom{k+d-1}{k}.
\]
Again making use of the hockey stick identity,
\[=
\frac{1}{2}
P
\binom{d+P}{P}
-\frac{1}{2}P\binom{d+P-1}{P-1}
\]\[+\binom{d+P-2}{P-1}
-\frac{1}{2}(P-1)\binom{d+P-1}{P}
+ \binom{d+P-2}{P-2}
\]

\[=
\binom{d+P-2}{P-1}
+\frac{1}{2}\binom{d+P-1}{P}
 + \binom{d+P-2}{P-2}
\]

\[=\binom{d+P-1}{P-1} + \frac{1}{2} \binom{d+P-1}{P}.
\]
Where we have used Pascal's identity several times. 

Then we address the second component, 
\[\frac{1}{2}\sum_{\substack{k=0\\ k=P\mod{2}}}^P
\left(\binom{k+d-1}{k}-\binom{k+d-3}{k-2}\right).
\]
This telescopes to 
\[=\frac{1}{2} \binom{d+P-1}{P}.
\]
Then, summing both components up and applying Pascal's identity yields
\[\binom{d+P-1}{P-1} + \frac{1}{2} \binom{d+P-1}{P}
+\frac{1}{2} \binom{d+P-1}{P}
=\binom{d+P}{P}.
\]

\subsection{Compression of the Radial Expansion}\label{sec:rankradial}

\begin{table}
\begin{center}
\begin{small}
\begin{tabular}{ c c c c c c c c c } 
 \toprule
 $d$ & 3 & 4 & 5 & 6 & 7 & 8 & 9\\
 \midrule
  $\frac{1}{r}$ & 1 & - & 2 & - & 3 & - & 4\\ 
  $\frac{1}{r^2}$ & - & 1 & - & 2 & - & 3 & - \\  
  $\frac{1}{r^3}$ & - & - & 1 & - & 2 & - & 3\\  
  $\frac{1}{r}e^{-r}$ & 1 & - & 2 & - & 3 & - & 4\\   
  $e^{-r}$ & 2 & - & 3 & - & 4 & - & 5\\  
  $re^{-r}$ & 3 & - & 4 & - & 5 & - & 6\\  
  $e^{-1/r}$ & 4 & 4 & 4 & 4 & 4 & 4 & 4\\
  $e^{-1/r^2}$ & 2 & 2  & 2 & 2 & 2 & 2 & 2\\
  \bottomrule
\end{tabular}
\end{small}
\end{center}
\caption{For a variety of different kernels in different dimensions, the value of $\mathcal{R}_k$ achievable in \eqref{eq:f_g}, independent of $P$. Dashes indicate that $\mathcal{R}_k$ was always found to be equal to its upper bound of $\lfloor \frac{P+k-2}{2}\rfloor$. By automatically finding these shorter expressions for the radial expansions when possible, we are able to change the $\lfloor \frac{P-k+2}{2}\rfloor$ term in \eqref{eq:num_multipoles} to a constant. The 2-term radial expansion for $e^{-r}$ is given in Table~\ref{table:exp_decomp}.}
\label{table:ranks}
\end{table}

\begin{table}
\begin{center}
\begin{small}
\begin{tabular}{ c c c }
\toprule 
\multicolumn{3}{c|}{$F_{k,i}(r)$}
\\ 
\midrule
 & $i=0$ & $i=1$  \\ 
\midrule
$k=0$ & $re^{-r}$  &  $-\frac{1}{3}e^{-r}$ 
\\
$k=1$ & $r^2e^{-r}$   &  $e^{-r}(\frac{-1}{5}r + \frac{-1}{5})$ 
\\
$k=2$ & $(\frac{1}{3}r^2 + \frac{1}{3}r^3)e^{-r} $  &  
$(\frac{-1}{7}r + \frac{-1}{42}r^2 + \frac{1}{42}r^3 + \frac{-1}{7})e^{-r}$ 
\\
\vdots & & \\
\bottomrule
\toprule  
\multicolumn{3}{|c}{$G_{k,i}(r')$} 
\\ 
\midrule
 & $i=0$ & $i=1$.  \\ 
\midrule
 $k=0$ & $1 + \frac{1}{6}r'^2 + \frac{1}{120}r'^4 + \frac{1}{5040}r'^6$  &   $r'^2 + \frac{1}{10}r'^4 + \frac{1}{280}r'^6$
\\
$k=1$ & $1 + \frac{1}{10}r'^2 + \frac{1}{280}r'^4 + \frac{1}{15120}r'^6$ & $r'^2 + \frac{1}{14}r'^4 + \frac{1}{504}r'^6$
\\
$k=2$ &  $1 + \frac{-1}{504}r'^4$ &  $r'^2 + \frac{1}{18}r'^4$\\
\vdots & & \\
\end{tabular}
\end{small}
\end{center}
\caption{Note that we are using $K(r)$ as shorthand for $K(|\textbf{r}'-\textbf{r}|)$. For $K(r)=e^{-r}$ we have $\mathcal{K}^{(k)}(r,r') = F_{k,1}G_{k,1} + F_{k,2}G_{k,2}$.}
\label{table:exp_decomp}
\end{table}

Here we remark on some beneficial properties of the term $\mathcal{K}^{(k)}(r',r)$ in our expansion. We define $\mathcal{R}_k$ to be the smallest number such that there exist functions $F_{k,i},G_{k,i}$ that satisfy
\begin{equation}\label{eq:f_g}
\mathcal{K}^{(k)}_p(r',r) = \sum_{i=1}^{\mathcal{R}_k} F_{k,i}(r)G_{k,i}(r'). \end{equation}
The motivation for focusing on this number is that it directly impacts the size $\mathcal{P}$ of our expansion, and hence the efficiency of our compression. In the case of $K(r)=1/r$ we have $\mathcal{K}^{(k)}(r',r)=r'^k/r^{k+1}$ and so
$F_{k,1}(r) = \frac{1}{r^{k+1}}, G_{k,1}(r') = r'^k,$ and  $\mathcal{R}_k = 1$.
For general kernels, we only have $\mathcal{R}_k \leq \lfloor\frac{p-k+2}{2}\rfloor$. 

However, it is possible for us to automatically detect when $F_{k,i},G_{k,i}$ exist so that $\mathcal{R}_k$ in~\eqref{eq:f_g} is smaller. Consider a kernel which satisfies the differential equation $K'(r) 
=
q(r)K(r)$,
where $q$ is a Laurent polynomial. In this case, the $m$th derivatives of the kernel will result in products of Laurent polynomials and the kernel itself, and hence the kernel may be pulled completely out of the double sum defining $\mathcal{K}^{(k)}_p$, yielding a binomial in $r'$ and $r$
\[\mathcal{K}^{(k)}_p(r',r) = 
K(r)
 \sum_{j}
\sum_{m}
r'^j
r^{m}A_{j,m},
\]
where the $A_{j,m}$ coefficients are computed based on the $\mathcal{T}_{jkm}^{(\alpha)}$ terms in the FKT expansion and the coefficients of the Laurent polynomial $q$. The sums over $j,m$ are finite and their range depends on the powers of the argument in the Laurent polynomial.
If the $A_{j,m}$ are rational, then a concise representation of the form~\eqref{eq:f_g} may be found in the following way:
(i) insert the coefficients $A_{j,m}$ into a matrix with rows and columns corresponding to the respective powers of $r$ and $r'$ in the binomial, (ii) perform a rank-revealing QR factorization~\citep{businger1965linear,chan1987rank} of the matrix but skip the normalization step so that all entries remain rational, and (iii) recover the functions $F_{k,i}$ from the coefficients in $Q$ and the functions $G_{k,i}$ from the coefficients in $R$. Because the entries remained rational, the rank found will exactly be the sought value of $\mathcal{R}_k$.

In our implementation, we automatically perform this computation of $\mathcal{R}_k, F_{k,i}(r)$, and $G_{k,i}(r')$ as a pre-computation when the given kernel satisfies $K'(r)=q(r)K(r)$ (this is indicated by a user-toggled flag). In order to keep entries rational during the factorization, we use a special Rational type within the Julia language rather than standard floating point operations. Although we find $\mathcal{R}_k = \lfloor \frac{p-k+2}{2}\rfloor$ for the squared exponential, we do see significant reductions in the size of the expansion for other kernels, notably Mat\'ern kernels. 
See Table~\ref{table:ranks} for some values of $\mathcal{R}_k$ for various kernels and dimensions, and Table~\ref{table:exp_decomp} for the functions $F_{k,i}$ and $G_{k,i}$ for the exponential kernel, for which $\mathcal{R}_k=2$.
\section{Additional Information for Experiments}
\subsection{Additional implementation details}
The major components of our implementation of the FKT are the tree decomposition and the population of the $s2m$ and $m2t$ matrices. Here we make additional comments on the latter, in which the novelty of the FKT is most manifest. 

To compute the $s2m$ and $m2t$ matrices requires (i) computation of hyperspherical harmonics, (ii) computation of the $m$th derivative of $K$ evaluated at $r$ in~\ref{eq:finaltrunc}, and (iii) computation of the $\mathcal{T}_{k,j,m}$ coefficients. (i) is a complicated expression of cosines, sines, and complex exponentials, which is daunting but doable with standard function calls in Julia, and our implementation aims to do this work in as vectorized a fashion as possible. (iii) is a similar task, although we remark that the coefficients do not depend on the data and can be stored once computed. (ii) is where auto-differentiation is leveraged\textemdash assuming the user has written their kernel in a format consumable by TaylorSeries.jl (e.g. \texttt{$kernel(r) = exp(-r^2)$}), then the tools from that package can compute any order derivative evaluated at any valid point. 
\subsection{Further error results for synthetic experiments}\label{sec:moreerror}
We performed the accuracy measurement experiment detailed in the section of the main text concerning synthetic experiments for many kernels in many dimensions. The results are presented in Table~\ref{table:error}. Notably the error is not significantly impacted by dimension (an observed increased accuracy with $d$ may be due to the experimental setup exploring relatively less of the space of function arguments), and shows consistent exponential decrease with the truncation parameter $p$.

We also remark that oscillatory kernels are known to have higher ranks for off-diagonal blocks. In kernel-independent FMMs which use factorizations of subblocks of the matrix, the result of attempting to compress a kernel matrix whose kernel has high-frequency oscillations is that little compression is achieved, accuracy is maintained, and runtime is comparable to a dense operation. For the FKT, the result would be consistent compression and runtime, but accuracy lost (since the interactions being compressed are not low-rank, as is assumed for the method). The user may, acknowledging this behavior of the kernel matrix, increase the truncation parameter so that accuracy is maintained at the cost of runtime, but our implementation of the FKT currently has no hooks to automatically detect the need for this. A wealth of literature exists for these kernels (c.f. \citet{dfmm}), and it is likely that an analogous extension of the FKT to incorporate considerations present in the directional FMM would improve performance with highly oscillatory kernels. 


\begin{table}
\begin{center}
    \begin{tabular}{ |l|l|l|l|l||l|l|l|l| }%
\multicolumn{9}{c}{Maximum Absolute Error} \\
      \hline%
      Kernel & 
      \multicolumn{4}{|c||}{$K(r)=e^{-r}$}
      &
       \multicolumn{4}{|c|}{$K(r)=\cos{r}/r$}
       \\
      \hline 
      Dim. & $3$ &  $6$ & $9$ &  $12$ & $3$ &  $6$ & $9$ &  $12$ \\ \hline%
      $p=3$ &  1.03e-2 &  1.02e-2 & 1.02e-2 &  1.02e-2   & 5.44e-2 &  3.07e-2 & 3.07e-2 &  3.06e-2 \\ %
      $p=6$ & 7.32e-4 &  6.78e-4 &  6.52e-4&  6.56e-4     & 7.60e-3 &  2.74e-3 & 2.01e-3 &  2.00e-3 \\ %
      $p=9$ &  5.48e-5 &  5.47e-5 & 5.40e-5 & 5.02e-5   & 7.68e-4 &  3.65e-4 & 2.34e-4 &  1.93e-4 \\ %
      $p=12$ & 4.62e-6 &  4.57e-6& 4.59e-6 & 4.31e-6    & 6.03e-5 &  3.23e-5 & 3.06e-5&  2.01e-5 \\ %
      $p=15$ & 4.25e-7 &  4.24e-7& 4.20e-7 & 3.98e-7    & 9.92e-6 &  3.48e-6 & 3.05e-6 &  2.59e-6 \\ %
      $p=18$ & 4.14e-8 & 4.14e-8 & 4.04e-8 &  4.04e-8    & 1.70e-6 &   5.23e-7 & 3.12e-7 &  2.82e-7 \\ \hline%
            \hline 
Kernel & 
      \multicolumn{4}{|c||}{$K(r)=(1+r^2)^{-1}$}
      &
       \multicolumn{4}{|c|}{$K(r)=e^{-r^2}$}
       \\
      \hline 
      Dim. & $3$ &  $6$ & $9$ &  $12$ & $3$ &  $6$ & $9$ &  $12$ \\ \hline%
      $p=3$ & 1.41e-2 &  1.41e-2& 1.41e-2 &  1.41e-2& 4.86e-2&  4.27e-2 & 2.95e-2 &  2.95e-2 \\ %
      $p=6$ & 2.17e-3 &  1.61e-3 &  1.11e-3&  1.11e-3& 9.42e-3 &  7.85e-3 & 4.91e-3 &  4.86e-3 \\ %
      $p=9$ &  1.58e-4 & 1.42e-4 & 1.39e-4&  9.51e-5 &9.32e-4 &  5.45e-4& 5.40e-4 &  3.87e-4 \\ %
      $p=12$ & 1.71e-5 &  1.54e-5 & 1.19e-5 &  8.29e-6& 4.80e-5 & 4.10e-5 & 4.10e-5&  2.64e-5 \\ %
      $p=15$ & 1.62e-6 &  1.27e-6 & 9.35e-7 & 9.18e-7& 2.29e-6 &  2.29e-6 & 1.96e-6 &  1.51e-6 \\ %
      $p=18$ & 1.39e-7 &  1.02e-7 & 7.69e-8 &  6.40e-8& 9.88e-8 & 9.88e-8 & 6.39e-8 &  4.07e-8 \\ \hline%
    \end{tabular}
\end{center}
\caption{Experimentally observed errors which are calculated for the $p=4$  FKT approximation by taking the maximum absolute error of the truncated expansion for 1000 randomly selected pairs of points $\textbf{r}',\textbf{r}$ satisfying $|\textbf{r}'|=1,|\textbf{r}|=2$.}
\label{table:error}
\end{table}

\subsection{Gaussian Processes}
A Gaussian Process (GP) is a distribution over functions
whose finite-dimensional marginal distributions 
are distributed according to a multivariate normal law.
That is, for any sample $f$ of a GP, 
and any finite set of inputs $\mathbf{X}$, we have
$
f(\mathbf{X}) \sim \N(\boldsymbol \mu_{\mathbf{X}}, \boldsymbol \Sigma_{\mathbf{X}})$,
for some mean vector $\boldsymbol \mu_{\mathbf{X}}$ and covariance matrix $\boldsymbol \Sigma_{\mathbf{X}}$.
In fact, analogous to the multivariate case, a GP is completely defined by its first and second moments:
a mean function $\mu(\cdot)$
and a covariance kernel $\kappa(\cdot, \cdot)$,
also known as a kernel.
In particular, if $f \sim \GP(\mu, \kappa)$ 
then for any finite collection of inputs $\mathbf{X}$,
\begin{equation}
f(\mathbf{X}) \sim \N(\mu(\mathbf{X}), \kappa(\mathbf{X}, \mathbf{X})),
\end{equation}
where $\kappa(\mathbf{X}, \mathbf{X})$ is the matrix whose $(i,j)^\text{nth}$ entry is $\kappa(\mathbf{X}_i, \mathbf{X}_j)$.
Fortunately, the posterior mean $\mu_p$ and posterior covariance $\kappa_p$ of a GP
conditioned on observations with normally-distributed noise have closed forms
and only require linear algebraic operations:
\begin{equation}
\label{eq:gp_posterior}
    \begin{aligned}
        \mu_p(\mathbf{X}_*) &= \mu(\mathbf{X}_*) + \kappa(\mathbf{X}_*, \mathbf{X}) \mathbf{\Sigma}_{\mathbf{X}}^{-1} (\mathbf{y} - \mu(\mathbf{X})),\\
        \kappa_p(\mathbf{X}_*, \mathbf{X}_*') &= \kappa(\mathbf{X}_*, \mathbf{X}_*') - \kappa(\mathbf{X}_*, \mathbf{X})  \mathbf{\Sigma}_{\mathbf{X}}^{-1} \kappa(\mathbf{X}, \mathbf{X}_*'),\\
    \end{aligned}
\end{equation}
where,
$\mathbf{\Sigma}_{\mathbf{X}} = k(\mathbf{X}, \mathbf{X}) + \sigma_y^2 \mathbf{I}$ and $\sigma_y$ is the standard error of the target $\mathbf{y}$.
We use the first formula to calculate the predictive mean of a GP for the oceanographic data in the main text using FKT.
For more background on Gaussian processes, see \citep{rasmussen2005gpml}.






\section{Existing Codes, GPU acceleration, and GP specific improvements}

\citet{deisenroth2015distributed} introduced
the robust Bayesian Committee Machine (rBCM)
which trains local GP "experts" on subsets of the data and combines their predictions.
All computations of rBCM can be carried out in a distributed fashion, but no constituent model is trained on the entire data.
\citet{de2017gpflow} introduced GPflow, a GP library based on accelerating variational inference procedures with GPUs via the TensorFlow framework.
GPyTorch is also a GPU-accelerated library, but is based on PyTorch and instead of variational inference, expresses all GP inference equations via MVMs \citep{gardner2018gpytorch},
relying on stochastic estimators to compute log-determinant and trace terms \citep{dong2017scalable}.

\end{document}